\documentclass[lettersize,journal]{IEEEtran}
\usepackage{amsmath,amsfonts,amsthm}
\usepackage{dsfont}
\usepackage{amssymb}
\usepackage{array}
\usepackage[caption=false,font=normalsize,labelfont=sf,textfont=sf]{subfig}
\usepackage{textcomp}
\usepackage{stfloats}
\usepackage{url}
\usepackage{verbatim}
\usepackage{graphicx}
\usepackage{cite}
\usepackage[ruled,linesnumbered]{algorithm2e}
\usepackage{algpseudocode}
\usepackage{enumitem}

\newtheorem{lemma}{Lemma}

\newtheorem{proposition}{Proposition}

\newtheorem{definition}{Definition}

\usepackage[table,xcdraw]{xcolor}
\usepackage{titlesec}
\usepackage[labelfont=bf,font=small]{caption} 
\definecolor{Gray}{gray}{0.9}
\def\BibTeX{{\rm B\kern-.05em{\sc i\kern-.025em b}\kern-.08em
    T\kern-.1667em\lower.7ex\hbox{E}\kern-.125emX}}

\titleformat{\subsubsection}[runin]{\normalfont\bfseries}{}{0em}{}[:]
\titlespacing{\subsubsection}{0pt}{\baselineskip}{\baselineskip}

\begin{document}

\title{\LARGE Attributing Responsibility in AI-Induced Incidents: A Computational Reflective Equilibrium Framework for Accountability}

\author{Yunfei Ge~\IEEEmembership{Student Member,~IEEE,} Ya-Ting Yang~\IEEEmembership{Student Member,~IEEE,} and Quanyan Zhu~\IEEEmembership{Senior Member,~IEEE.}
\thanks{The authors are with the Department of Electrical and Computer
Engineering, Tandon School of Engineering, New York University, Brooklyn,
NY 11201 USA (e-mail: yg2047@nyu.edu; yy4348@nyu.edu; qz494@nyu.edu).}
\thanks{This work has been presented in internal multidisciplinary venues spanning engineering, psychology, and philosophy, and has benefited from comments and feedback provided by participants as well as anonymous reviewers.}}




\maketitle

\begin{abstract}

The pervasive integration of Artificial Intelligence (AI) has introduced complex challenges in the responsibility and accountability in the event of incidents involving AI-enabled systems. The interconnectivity of these systems, ethical concerns of AI-induced incidents, coupled with uncertainties in AI technology and the absence of corresponding regulations, have made traditional responsibility attribution challenging. To this end, this work proposes a Computational Reflective Equilibrium (CRE) approach to establish a coherent and ethically acceptable responsibility attribution framework for all stakeholders. The computational approach provides a structured analysis that overcomes the limitations of conceptual approaches in dealing with dynamic and multifaceted scenarios, showcasing the framework's traceability, coherence, and adaptivity properties in the responsibility attribution process. We examine the pivotal role of the initial activation level associated with claims in equilibrium computation. Using an AI-assisted medical decision-support system as a case study, we illustrate how different initializations lead to diverse responsibility distributions.  The framework offers valuable insights into accountability in AI-induced incidents, facilitating the development of a sustainable and resilient system through continuous monitoring, revision, and reflection.

\end{abstract}


\section{Introduction}

In today's context, Artificial Intelligence (AI) has been integrated into systems with critical functionalities. Cyber-physical systems, such as smart health systems, cloud-enabled critical infrastructure, and connected vehicles, are quintessential examples. The wide adoption of AI, however, introduces a range of issues, encompassing safety concerns in autonomous driving, the reliability of AI-driven medical diagnosis and treatment, and the unforeseen outcomes arising from medical recommendation systems. The recent White House executive order \cite{biden2023executive} places the highest priority on governing the development and use of AI in a safe and responsible manner. Nonetheless, as an emerging technology with a black-box nature, AI introduces uncertainties that cannot be easily identified, assessed, and attributed by traditional approaches. Concerns arise due to the absence of mechanisms for responsibility attribution in incidents caused by AI systems. The lack of accountability can result in hesitancy in the pervasive adoption of AI, thereby limiting the growth and the impact of the technology. To this end, there is a need to establish an accountability framework for AI-enabled systems, facilitating the future development of safe, secure, and trustworthy AI.

Accountability involves recognizing and accepting responsibility for the potential outcomes of one's actions. In the context of AI-related incidents, it revolves around identifying responsibility for failures in AI systems \cite{10457538}. For example, AI-enabled medical decision-support systems rely on AI algorithms to provide medical diagnoses and recommendations.  Typically, a physician administers a medical test to a patient, and the results are input into the system. The doctor, relying on the AI's suggested results, prescribes medication to the patient. If medical incidents occur due to the patient following the doctor’s prescription, various parties could be accountable.  Possible points of accountability include the physician's actions during the operation, potential errors in AI development, the doctor's dependence on AI recommendations, and the patient's adherence to the treatment plan. Determining how to allocate responsibility among all involved parties proves to be a challenging task.

\begin{figure}[!t]
    \centering
  \includegraphics[width=1\linewidth]{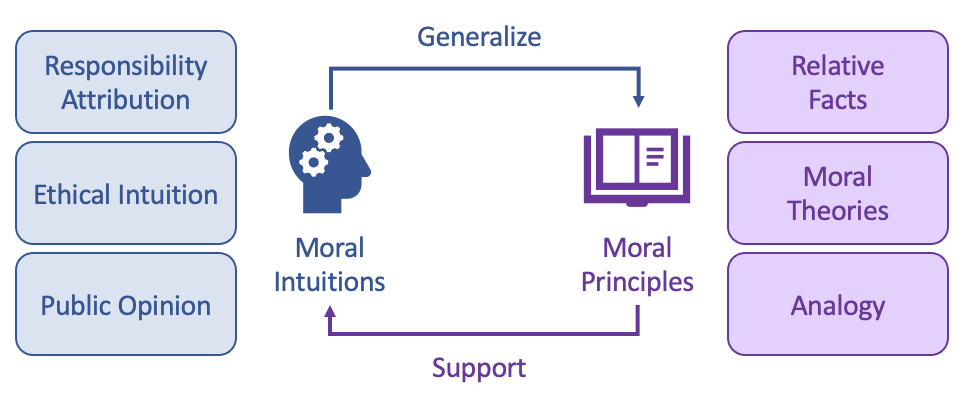}
  \caption{General process of reflective equilibrium: The process involves formulating a set of moral intuitions, applying evidence and moral principles to support the claim, and iteratively revising and reflecting on intuitions to establish a coherent equilibrium.}
  \label{fig:iterate}
  \vskip -2mm
\end{figure}

Attributing AI responsibility has the following challenges. The first challenge arises from the interconnected and complex nature of AI systems, where inherent vulnerabilities in each component can contribute to potential malfunctions, making responsibility attribution to individual components a complex task \cite{ge2022accountability}. The second challenge involves uncertainties in AI, both in the technology itself and in regulations governing AI. The AI system provides a diagnosis, but the reasoning behind the recommendation is not transparent or explainable, creating difficulties in determining accountability without clear examination tools. Additionally, there is a lack of legal regulations and guidelines outlining expected conduct, responsibilities, and ethical principles for practitioners when utilizing AI in their practice. Furthermore, new challenges also arise from ethical constraints. One example is the allocation of limited medical resources by an AI decision-support system, prioritizing certain groups of patients over others. Determining accountability becomes complex when such decisions result in inadequate medical resources for underprivileged patients.

To address these inherent challenges, we seek to establish a coherent responsibility attribution framework among stakeholders supported by justified beliefs. Drawing inspirations from moral and political philosophy, particularly associated with the work of John Rawls \cite{rawls2017theory, daniels2003reflective, daniels1979wide}, we employ Reflective Equilibrium (RE) as our approach. RE begins with moral intuitions or judgments regarding specific cases or general principles, representing immediate and unreflective responses to given situations. Subsequently, evidence and theoretical principles from philosophical frameworks or other theories are introduced. A comparison is then made between these intuitions and the theoretical principles, whether new or existing. In the presence of inconsistencies or conflicts, adjustments are made to establish a coherent and balanced set of principles and intuitions that mutually support each other. The iterative process, as illustrated in Fig.~\ref{fig:iterate}, eventually leads us to consistency in ethical responsibility attribution. In the context of AI-induced incidents, many inconsistencies arise due to the absence of applicable theory conflicting with human intuition.

Applying RE to AI-induced incidents, we first establish a framework utilizing reflective equilibrium to assess and assign responsibility to achieve a coherent and ethically acceptable equilibrium for all stakeholders.  Initially, we formally state all possible responsibility attribution principles for AI-enabled systems involving multiple parties. Subsequently, we apply evidence and moral principles to support the responsibility attribution accordingly.  We refer to the intuitions and principles as claims. 
To find the equilibrium, we formulate the problem as a constraint satisfaction problem using coherence theory and employ a computational approach, defining our solution as Computational Reflective Equilibrium (CRE). This computational framework provides a structured analysis that overcomes the limitations of conceptual approaches in dealing with dynamic and multifaceted scenarios. 
To enhance results, we incorporate hypothesis testing and gather feedback from those involved in or contributing to accountability measures to determine the initial activation level for each claim. We showcase that different initializations would lead to diverse responsibility distributions in the equilibrium. 

The computational approach aims to achieve a coherent and ethically justifiable equilibrium that minimizes conflicts and maximizes support, achieving consistency among the stakeholders. By employing computation, our framework not only enhances clarity in ethical reasoning but also demonstrates adaptability to diverse ethical contexts, promoting transparency and coherence in responsibility attribution. Computational Reflective Equilibrium (CRE) facilitates a dynamic balance among conflicting ethical principles, obligations, and evidence, offering context-sensitive solutions. To ensure transparency, it is imperative to communicate the principles and accountability standards to all stakeholders. Moreover, continuous monitoring and revisiting the equilibrium for necessary adjustments in case of discrepancies are also crucial components of the process. We use a medical decision-support system as a case study to illustrate the framework and provide computational results to highlight its properties.

The remainder of the paper is organized as follows. In Section II, we delve into the related work within the field. 
Section III outlines the computational reflective equilibrium (CRE) framework. In Section IV, we discuss the properties of the proposed method. The performance of CRE through simulations is presented in Section V. Lastly, we draw our conclusions in Section VI.

\section{Related Work}

\subsection{AI Responsibility}

The concept of responsibility has several connotations. In this study, our focus is on backward-looking responsibilities, specifically understanding responsibility as blameworthiness \cite{van2012problem}. A primary challenge in assigning responsibility in AI-induced incidents lies in identifying the morally responsible agent. While humans are traditionally associated with agency and accountability for their actions, it remains unclear whether AI technologies can be considered responsible agents themselves. Extensive discussions about the moral agency of AI have taken place in the ethics of computing and robot ethics literature \cite{cervantes2020artificial, hage2017theoretical, coeckelbergh2010moral, gunkel2020mind}. Depending on the moral theory adopted, one may choose to include or exclude AI as the responsible agent. Currently, there is no legal regulation treating AI as a responsible agent, but as technology evolves, future perspectives may emerge \cite{van2018we}. In our framework, we can accommodate both arguments as long as they are supported by the justified belief system in the end.

Once we identify the responsible agents, another challenge arises in how to attribute responsibilities among these agents, a problem often referred to as the ``many hands'' problem \cite{thompson1980moral}. Several works have focused on addressing this challenge. Duch et al. \cite{duch2015responsibility} applied weighted votes to distribute responsibility among multiple collective decision-makers, but this approach may not be applicable to technologies like AI, as we cannot collect their votes. Taddeo et al. \cite{taddeo2018ai} proposed the conceptual idea of distributed responsibility to address responsibility attribution in the AI context; however, practical solutions for its implementation are lacking. Other detailed research has focused on determining ``who did what at when'', seeking to establish causality and the role of each agent to attribute responsibility on a case-by-case basis \cite{liu2022blame, coeckelbergh2020artificial}. While there is no definitive correct answer to responsibility attribution,  what is notably missing is a general model addressing how to attribute responsibility in AI, supported by reasoning and justified outcomes, preferably in a computationally supported manner.

\subsection{Reflective Equilibrium}
In applied ethics, the reflective equilibrium (RE) model has been proposed as an approach to reconcile a pluralism of ethical views. RE enables decision-makers to achieve coherence and consistency in ethical reasoning by balancing and refining moral intuitions, principles, and theories \cite{daniels2003reflective}, through mutual adjustment and conceptual reflection. Typically, this process begins with expert intuitions and allows for both the revision of specific judgments and the underlying concepts themselves. It has demonstrated success in various areas, including justice \cite{rawls2017theory}, moral decision-making \cite{yilmaz2016coherence}, and public reasoning \cite{brandstedt2020rawlsian}. While Doorn \cite{doorn2012exploring} explored responsibility rationales in research and development using RE, their contribution was primarily descriptive. Yilmaz et al. developed a domain-specific language to establish coherence-governed models for ethical decision-making.  

\noindent\textbf{Discussion on Differences between RE and CRE}\\
The proposed computational reflective equilibrium (CRE) adapts the RE approach to settings involving multiple stakeholders and complex, dynamic scenarios (such as AI responsibility attribution), with two key differences. First, CRE allows the initial activation of claims to be informed by public preferences and empirical data, as well as expert judgment. Second, the revision process is currently limited to updating the acceptance or rejection of predefined claims and their network of support and conflict relations, rather than full-scale conceptual revision. While this limits some of the philosophical richness of RE, it allows for computational tractability, transparency, and stakeholder inclusion.
Building on the literature of RE and its associated AI frameworks, the CRE framework retains the iterative, coherence-seeking character of RE and can be continuously updated as new claims or evidence arise. Future work may explore how richer forms of conceptual revision might be integrated into the computational process.

\section{Computational Reflective Equilibrium to Accountability}

\begin{figure}[!t]
    \centering
  \includegraphics[width=1\linewidth]{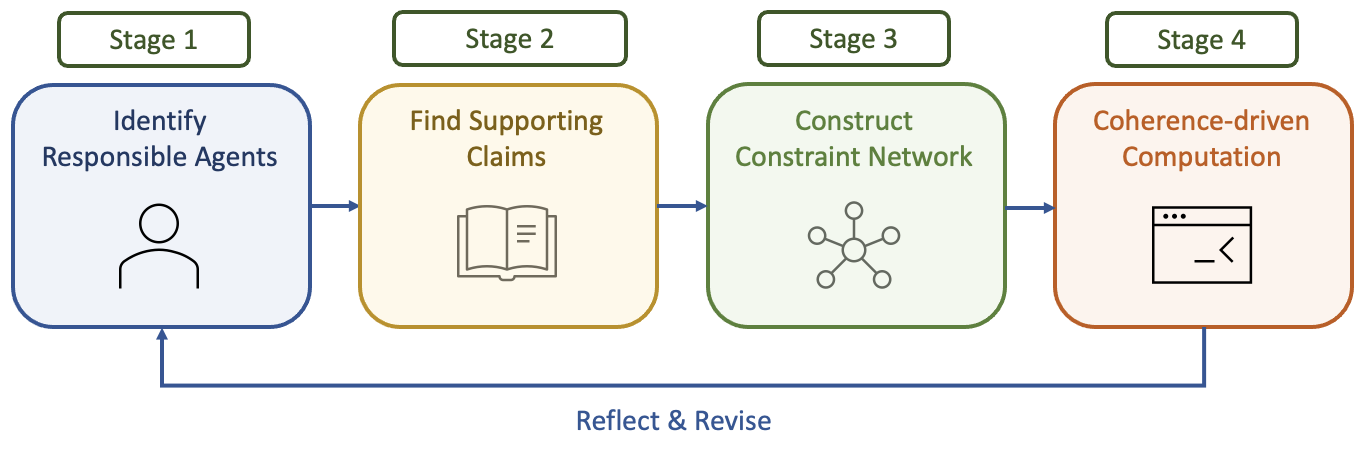}
  \caption{Workflow of computational reflective equilibrium approach to responsibility attribution.}
  \label{fig:overall}
\end{figure}

In this section, we describe the computational reflective equilibrium computation (CRE) process for achieving coherent and ethical responsibility attribution in AI-induced incidents. The basic workflow is depicted in Fig.~\ref{fig:overall}. The committee that implements CRE is referred to as the implementation committee (IC). Here, we focus on the example of a medical decision-support system for diagnosis and treatment to elaborate on the proposed approach.

\subsection{Stage $1$: Identify Responsible Agents}

In the initial stage, IC formulates hypotheses for accountability attribution that align with intuitions and identify all potential responsible agents involved in the AI incident \cite{cath2016reflective}. This process includes recognizing all relevant parties that can contribute to the final consequences. Importantly, the selection of responsible agents should not rest solely on the judgment of IC. To ensure rigor and legitimacy, this process must be anchored in legal and ethical precedents, informed by stakeholder mapping, and accompanied by transparent documentation of inclusion and exclusion criteria. Such practices reduce arbitrariness and strengthen the credibility of the attribution results.

As discussed earlier, there is deliberation on whether to include AI itself as a moral agent. This work asserts that such inclusion is justifiable as long as it is supported by theories and evidence. This highlights the adaptivity of the CRE framework, which accommodates diverse claims as initial input from various perspectives. At the same time, the IC should also consider opposing claims where the agent is not responsible. This comprehensive approach ensures a balanced consideration, incorporating both affirmative and negative perspectives for a thorough analysis of responsibility attribution.

For illustrative purposes, this work considers a scenario in which the AI decision-support system generates a diagnosis, and the doctor adopts a treatment plan based on the recommendation, ultimately resulting in an incorrect treatment for the patient. The focus is specifically on the interaction between the AI system and the doctor, excluding other factors such as the manufacturer producing the product, the physician collecting data, or the patient failing to adhere to the prescription. For a more comprehensive analysis, all these factors can be incorporated into the graph. In this instance, the following initial claims regarding responsibility attribution can be outlined, as depicted in Fig.~\ref{fig:1}.

\begin{figure}[!t]
    \centering
  \includegraphics[width=0.8\linewidth]{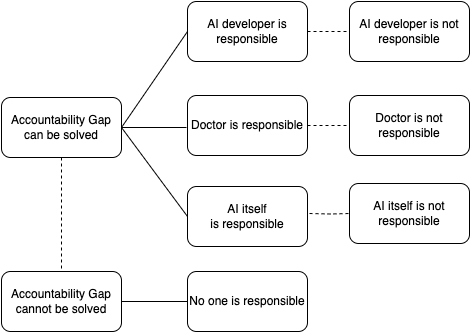}
  \caption{Initial claims about responsibility attribution in AI medical decision-support systems. Solid lines indicate supporting claims, while dotted lines represent conflicting claims.}
  \label{fig:1}
\end{figure}

\subsection{Stage $2$: Find Supporting Claims}

After identifying the initial claims, the next step involves formulating an initial set of theoretical principles or relevant facts that generalize or support these initial claims. For each initial claim, IC can seek supportive claims in the following directions \cite{thagard2012cognitive}:
\begin{enumerate}[leftmargin=*]
    \item Relative Facts: This entails identifying the domain knowledge or evidence that supports the initial claim. Such supportive claims can often be quantified through evaluations like hypothesis testing.
    \item Moral Concerns and Theories: This direction involves finding the generalized moral theory to ethically support the claim.
    \item Analogy: Drawing parallels with similar cases in other areas serves as an analogy to support the claim. For example, referring to robot-induced incidents can provide insights into AI-induced incidents as a support factor.
    \item Opposite opinions: In addition to identifying supportive claims, it is crucial to list possible opposing opinions to these claims, forming a more comprehensive analysis.
\end{enumerate}
By considering these claims, we can construct a coherent explanation of accountability that is acceptable and traceable. The details about the supportive claims for the illustrative AI medical decision-support system case can be found in Fig.~\ref{fig:network}.

It is important to highlight that the claims, including initial claims and their supportive claims, should not be arbitrary but rather directly or analogically related to the relevant domain. Holding any unrelated party accountable in the case of AI-induced incidents would be unreasonable. 
To mitigate arbitrariness in identifying supporting claims, candidate claims should be systematically derived from established sources such as professional codes of ethics, regulatory guidelines, and empirical studies. Structured stakeholder engagement (e.g., expert panels or participatory methods) can further help surface and validate claims, ensuring that the process is inclusive, transparent, and robust.
Moreover, several philosophers have suggested that RE is best interpreted as a hybrid of foundationalism and coherentism \cite{cath2016reflective}. This suggests that IC should adhere to the following constraints for the claims.
\begin{definition}[Claim Constraints]
The claims for computational reflective equilibrium should adhere to the following conditions:
\begin{enumerate}[leftmargin=*]
    \item The claim is directly or analogically related to the relevant domain,
    \item The claim must already be justified to some degree,
    \item The justification is non-inferential, based directly on the mere fact of the subject's believing or intuiting as they do.
\end{enumerate}
\end{definition}

\subsection{Stage $3$: Construct Constraint Network}

In the third stage, this work attempts to formulate the constraint network that describes the relations among all claims. The network can be represented as a graph $G=(V,E)$, where $V$ represents a finite set of claims and $E\subseteq V\times V$ represents the set of relations between two claims. These relations could be supports or conflicts. The support between the two components strengthens one's belief in certain claims. However, conflicts are likely to arise between one's initial claims and the supportive claims that aim to account for those claims. Such conflicts may arise within the initial claims or within the supportive claims, reflecting the challenge of determining accountability in complex systems like AI. 

Because the specification of support and conflict relations has a decisive influence on the equilibrium outcome, it is essential that these relations be justified explicitly. Empirical studies, such as research on human–automation blame attribution \cite{lei2021should}, can provide evidence-based grounding, while normative reasoning from ethical theory can offer additional justification \cite{kagan2018normative,yazdanpanah2023reasoning}. To further ensure reliability, sensitivity analysis can also be conducted across alternative specifications of support and conflict relations. This practice helps reduce the risk of bias or misplaced confidence in any single network configuration.

\begin{figure*}[!ht]
    \centering
    {\includegraphics[width=\textwidth]{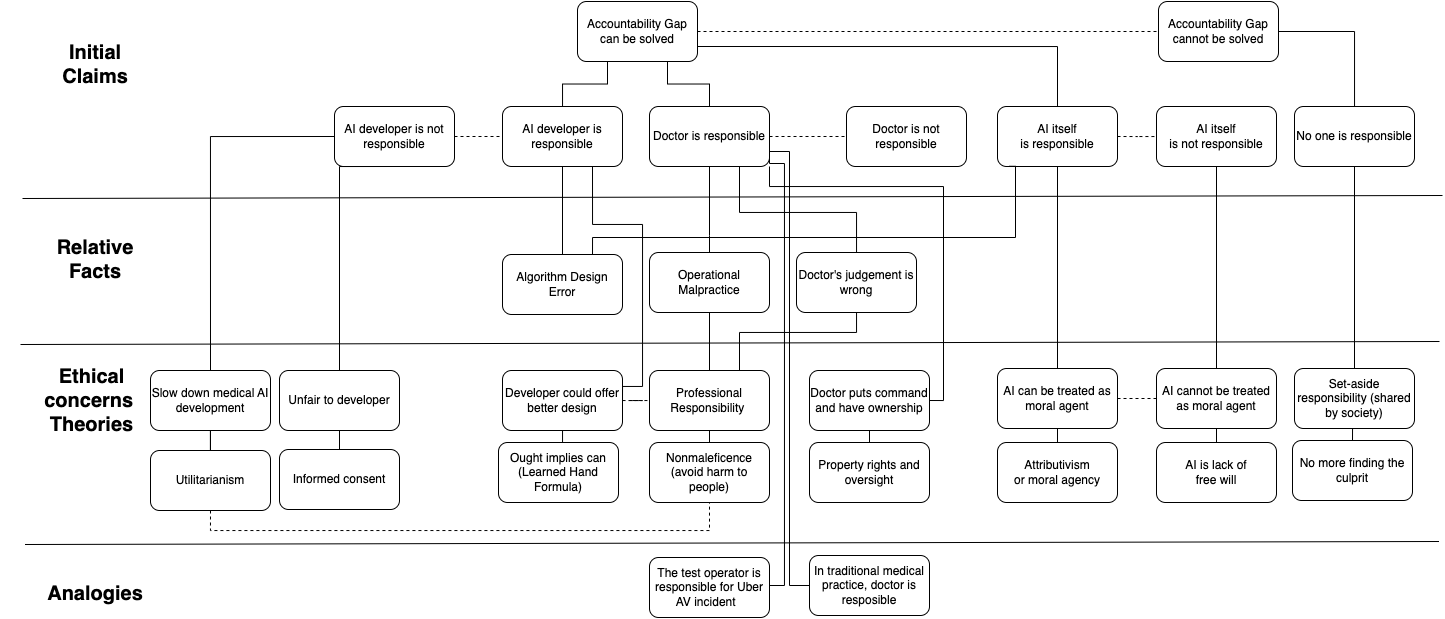}}
    \caption{Constraint network for responsibility attribution in the AI medical decision-support system case. Solid lines indicate supporting claims with positive constraints, while dotted lines represent conflicting claims with negative constraints.}
    \label{fig:network}
\end{figure*}

Inspired by the notion of constraint satisfaction studied in computer science, we can formalize support (coherence) and conflict (incoherence) relations as positive and negative constraints, respectively \cite{thagard1998coherence}.
\begin{definition}[Satisfaction conditions]
The coherence conditions for the positive and negative constraints are:
    \begin{enumerate}
        \item[(C1)] A positive constraint $(u,v)\in C^+$ is satisfied if and only if the elements $u$ and $v$ are both rejected or both accepted; 
        \item[(C2)] A negative constraint $(u,v)\in C^-$ is satisfied if and only if the element $u$ is accepted and $v$ is rejected, or conversely.
    \end{enumerate}
\end{definition}

In the illustrative AI medical decision-support case, the full constraint network with initial responsibility attribution claims and supportive claims are illustrated in Fig.~\ref{fig:network}, where solid lines indicate supporting claims with positive constraint and dotted lines represent conflicting claims with negative constraint.

\subsection{Stage $4$: Coherence-driven Computation}

Once we establish these positive and negative constraints, our goal is to identify a set of claims that exhibit coherence, forming our responsibility attribution result with supporting explanations. The process of computing reflective equilibrium can be understood as a coherence problem, seeking to satisfy as many constraints as possible. Formally, we can define the coherence problem for finding RE as follows.

\begin{definition}[Coherence problem]
    Given a graph $G = (V,E)$ with vertex set $V$ represents a finite set of elements (belief, propositions, evidence, etc.) and edge set $E\subseteq V\times V$ that partitions into positive constraints $C^+$ and negative constraints $C^-$, i.e., $C^+\cup C^- = E$ and $C^+\cap C^- = \emptyset$. Each constraint $(u,v)\in E$ is associated with a weight $w(u,v)\in\mathbb{R}$ that indicates the weight of the constraint. The coherence problem aims to partition the vertices $V$ into two sets, $A$ (accepted) and $R$ (rejected) such that the sum of the weights of the satisfied, given by:
    \begin{align}
        \max_{ A \cup R = V} \quad W(A,R) = \sum\limits_{\substack{(u,v)\in C^+ \\ \text{satisfy C1}}} w(u,v) + \sum\limits_{\substack{(u,v)\in C^-\\ \text{satisfy C2} }} w(u,v).
    \end{align}
\end{definition}

The challenge of coherence is recognized as NP-hard; however, practical approximations can be reliably employed. One computational approach to approximating coherence involves leveraging connectionist algorithms, such as neural networks. In neural science, this model aligns with the problem of maximizing harmony in Hopfield networks \cite{van2012intractability}. Under this interpretation, the focus shifts from partition identification to assigning each element an activation level $a:V\mapsto [-1,1]$ in a manner that maximizes the subsequent function:
\begin{align}
    \max_a H(a) = \sum_{u,v\in V} a(u) a(v) w(u,v).
\end{align}

\begin{lemma}
    Given a constraint network $G=(V,E)$ and weight function $w: E\mapsto \mathbb{R}$, there exists an activation level $a^*:V\mapsto\{+1,-1\}$ that maximizes $H(a)$.
\end{lemma}

The detailed proof of the lemma can be found in \cite{van2012intractability}. The lemma suggests that at least one partition can be identified to achieve reflective equilibrium. It's worth noting that the reached equilibrium offers the most fitting explanation only for the particular context and situation. Individuals with different preferences may reach different equilibria, each guided by distinct principles and judgments.

Algorithm~\ref{alg:compute} outlines the procedure for updating the activation levels using connectionist approximation. In this algorithm, $net^t_u=\sum_v \hat{w}(v,u)a^t(u)$ represents the net input to element $u\in V$. 
If $(v,u)\in C^+$, we assign $\hat{w}(v,u) = w(v,u)$; else if $(v,u)\in C^-$, we let $\hat{w}(v,u) = -w(v,u)$.
The parameter $\gamma\in(0,1)$ introduces a decay parameter, reducing the activation level of each element in each iteration. The values of $M$ and $m$, usually set to $M=1$ and $m=-1$, define the maximum and minimum activation ranges. The algorithm produces a decision for claims $u\in V$ where $a^*(u)>0$, indicating acceptance of the claim; otherwise, it is rejected. The accepted claims constitute a coherent and explainable set that aligns closely with ethical considerations. To determine the final responsibility attribution, we examine the accepted sets to identify which involved party should be accountable, along with the supportive claims that bolster the assignment. 

{\renewcommand{\arraystretch}{1.5}
\setlength{\textfloatsep}{2pt}
\begin{algorithm}[t]
\caption{Computational Reflective Equilibrium Algorithm}\label{alg:compute}
\KwIn{Constraint graph $G=(V,E)$, positive constraint set $C^+$, negative constraint set $C^-$, weight function $w:E\mapsto \mathbb{R}$, decay parameter $\gamma\in(0,1).$}
\vskip 1mm
Initialization: assign $a^0(u)$ for each element $u\in V$ \;
\Repeat{activation levels reach equilibrium}{
Update the activation of all the elements parallel:
\begin{align*}
    a^{t+1}(u) =  &a^{t}(u)(1-\gamma)\\
    &+\begin{cases}
        net^t_u\cdot (M - a^{t}(u)) & \text{if } net^t_u>0,\\
        net^t_u\cdot (a^{t}(u)-m) & \text{otherwise.}
    \end{cases}
\end{align*}
}
\vskip 1mm
\KwResult{If $a^*(u)>0$, the element $u$ is accepted. Otherwise, $u$ is rejected.} 
\end{algorithm}}

\subsection*{Initial Activation Level}

As mentioned earlier, the equilibrium is influenced by individuals' initial beliefs in the given claim. In this context, we argue that the initial activation level of each claim, $a^0(u)$ for all $u\in V$, effectively represents the initial belief associated with that specific claim. The following properties in cognitive science provide insights into the effect of the initial activation level.
\begin{itemize}[leftmargin=*]
    \item \textit{The Rich Get Richer Effect} \cite{grossberg1976adaptive}: Elements $u\in V$ that possess slight initial advantages, be it in terms of their external inputs or initial activation values, tend to magnify this advantage when compared to their competitors, specifically those denoted as $v\in V$ with a negative constraint $(u,v)\in C^-$.

    \item \textit{Resonance Effect} \cite{grossberg1987competitive}: If element $u\in V$ and $v\in V$ have mutually excitatory connections (positive constraints $(u,v)\in C^+$), then once one of the elements becomes active, they will tend to keep each other active. 
\end{itemize}
Building upon these effects, we present the following proposition.
\begin{proposition}
For an input claim $u\in V$ with higher initial activation levels $a^0(u)$, they possess greater priority and a higher likelihood of sustaining activation throughout the equilibrium, denoted as $a^*(u)= +1$.
\end{proposition}

To quantitatively assess the initial activation levels of each input claim, we can categorize the claims into two sets. The first set comprises ethical, theoretical, or regulatory claims that are either unquantifiable or untestable, and the second set includes claims that can be evaluated through scientific testing or evidence gathering. We present the quantitative investigation methods for these two types of claims in the following sections.

\subsubsection{Public Preference}
Individuals often hold diverse intuitions or perceptions regarding a single ethical claim. For instance, some may believe that AI developers should be held accountable if algorithms misinform doctors, while others may argue that such accountability could discourage innovation and slow the advancement of AI technologies. These disagreements illustrate the challenge of translating ethical intuitions into quantitative values, since there is no direct metric for evaluating the relative importance of the claims.


For claims that cannot be tested empirically, initial activation levels are therefore assessed through subjective judgments. In the proposed framework, individuals have the flexibility to initialize values based on their personal judgments, where $a^0(u)=-1$ indicates strong disagreement, $a^0(u)=0$ denotes neutrality, and $a^0(u)=1$ signifies strong agreement.
Public preferences, elicited through surveys, can serve as one method for setting the initial activation levels of such claims. This approach reflects a democratic commitment to ensuring that responsibility attribution aligns with the values and intuitions of the stakeholders involved.
\begin{proposition}
    Given the claim $u\in V$, Let $f_u(x)$ be the probability distribution function representing the distribution of preferences $x\in[-1,1]$ for claim $u$ across the population. The initial activation level of claim $u$ can be represented as
    \begin{align}
        a^0(u) = \int_{-1}^1 x f_u(x)dx.
    \end{align}
\end{proposition}


A further consideration concerns the scope of the surveyed population \cite{rea2014designing}. Depending on the application, surveys may target local communities directly impacted, national or regional populations that share legal and cultural contexts, or global populations to emphasize universality. Each choice has distinct implications: local surveys provide contextual sensitivity, while broader populations support generality and fairness. The scope must therefore be explicitly justified, and the societal trade-offs of each choice must be made transparent.

In practice, the activation value can be derived by assessing people's acceptance of a specific claim through surveys, seeking expert opinions, referring to regulatory suggestions, or consulting academic research. In many contexts, a hybrid approach that combines public and expert input provides a more balanced and legitimate foundation for initialization. The integration of public preference aims to translate the intuitions and beliefs of individuals into practical data, thereby contributing to the balancing process. This approach aligns with the framework in medical ethics \cite{beauchamp2003methods}, emphasizing autonomy, beneficence, nonmaleficence, and justice principles, all of which are closely related to the well-being and perspectives of the people involved. However, it is worth noting that the quantitative survey should have a sufficient sample size to draw a conclusion and should be designed carefully to explicitly account for factors such as population scope, fairness in representation, and the balance between public and expert input to eliminate potential bias. Otherwise, the results could incline towards particularity rather than universality \cite{dunn2012toward}.

\subsubsection{Quantitative Investigation}

For claims of the second type, their initial activation values can be quantified through evaluations. The key lies in investigating whether empirical evidence supports the claim. One appropriate approach for determining these values is through hypothesis testing. In line with the framework outlined in \cite{ge2022accountability}, we employ hypothesis testing to illustrate the computation of initial activation levels for individual claims. It's worth noting that alternative evaluation methods are also acceptable.

Consider a claim in the constraint network that states that ``the developed algorithm has no error''. Assume the claim has an authentic binary type $\Theta = \{0,1\}$, where $\theta=1$ means the statement is true and $\theta=0$ means the statement is false. Since the algorithm performance cannot be directly accessed but only can be tested through operations, we consider the following. The responsible party of this statement (algorithm developer) sends a message to the public stating whether the statement is true or not, denoting $m\in\mathcal{M} = \{0,1\}$. The doctor decides how to incorporate the algorithm into medical practices, denoted as an operation function $\delta: \mathcal{M} \mapsto \mathcal{D}$. The final performance follows a distribution $p_y(y;\theta,\delta(m))$. Consider a sequence of repeated but independent investigation observations $Y^k=\{y_1, y_2, \cdots, y_k\}$, $k\in \mathbb{N}$. Hypothesis $H_0$ is set to be the case when the observations follow the anticipated performance given the received message $m$ and $H_1$ otherwise. Depending on whether $H_0$ or $H_1$ holds, each observation $y_i$ follows the respective distribution:
    \begin{align}
            &H_0: \, y_i \sim f_m(y|H_0) = p_y(y;\theta = m,\delta(m)), \\
            &H_1: \, y_i \sim f_m(y|H_1) = p_y(y;\theta = \neg m,\delta(m)).
    \end{align}
    
The optimal Bayesian investigation rule is based on the likelihood ratio test (LRT). LRT provides the decision rule that $H_1$ is established when $ L(Y^k)$ exceeds a defined threshold value $\tau\in\mathbb{R}$; otherwise, $H_0$ is established. It can be formulated by 
    \begin{align}
         L(Y^k)=\prod_{j=1}^k\frac{p_y(y_j;\delta(m)|\theta=\neg m)p_\theta(\theta = \neg m)}{p_y(y_j;\delta(m)|\theta = m)p_\theta(\theta = m)} \mathop{\gtreqless}_{H_0}^{H_1} \tau.
    \label{LRT}
     \end{align}
where $p_\theta$ is the prior distribution of the claim, which indicates the prior probability that the developed algorithm has an error or not.
Assume the investigation cost is symmetric and incurred only when an error occurs. In the binary case, the optimum decision rule will consequently minimize the error probability, and the threshold value $\tau$ in LRT \eqref{LRT} will reduce to $\tau = \Pr(H_0)/\Pr(H_1)$, where $\Pr(H_i)_{i=0,1}$ is the prior distribution of the hypothesis, indicating the reputation of the developer.

\begin{definition}[Claim Authenticity]
\label{accountability}
For a given claim $u\in V$, given an investigation rule,  i.e., the threshold $\tau_u$, claim authenticity $P_A \in [0,1]$ is defined as the probability of correct establishment of hypothesis $H_1$ based on the observations $Y^k$ and message $m$, which is given by
\begin{align}
    P_A(\tau_u)= \int_{\mathcal{Y}_1} f_m(Y^k|H_1)dy^k,
    \label{eq:padef}
\end{align}
where $\mathcal{Y}_1$ is the observation space in which $\mathcal{Y}_1 = \{Y^k: L(Y^k)\geq \tau_u\}$.
\label{def:1}
\end{definition}

{\renewcommand{\arraystretch}{1.5}
\setlength{\textfloatsep}{2pt}
\begin{algorithm}[t]
\caption{Claim Authenticity Investigation Procedure}\label{alg:acc}
\KwIn{Original claim $u$, received message $m$, operation function $\delta(m)$, and reputation regarding the claim $\pi_u$\;}
\vskip 1mm
Establish hypotheses $H_0$ and $H_1$ based on received message $m$\;
Determine the investigation rule $\tau_u$ in LRT based on the reputation  $\pi_u$\;
Perform $k$ times of \textit{i.i.d.} investigation tests and record the observations $Y^k$\;
Compute the claim authenticity according to \eqref{eq:padef} \;
\vskip 1mm
\KwResult{Claim Authenticity $P_A(\tau_u)$. }
\end{algorithm}}

The claim authenticity Investigation investigation procedure is summarized in Algorithm~\ref{alg:acc}. $P_A(\tau_u)$ indicates the probability that the claim holds within the constraint network. This value is a probability that lies within the range $P_A(\tau_u)\in[0,1]$. To effectively represent the initial activation level, it is necessary to normalize the outcomes of the claim authenticity assessment to fit within the initial activation level region $a^0(u)\in[-1,1]$.
\begin{proposition}
    Given the claim authenticity investigation result $P_A(\tau_u)\in[0,1]$ of a claim $u\in V$, the initial activation level of the element is
    \begin{align}
        a^0(u) = 2P_A(\tau_u)-1.
    \end{align}
\end{proposition}
The individual claim investigation process integrates one's prior beliefs about the claim with the testing performance acquired through investigative efforts. This combination of beliefs and observations aims to provide a more comprehensive interpretation, thereby guiding the allocation of accountability to the relevant element.

\subsection{Result and Revision}

After determining the initial activation levels for all input claims, the reflective equilibrium is computed using the approximation algorithm outlined in Algorithm~\ref{alg:compute}. The algorithm returns a set of coherent and mutually supporting claims. By inspecting this set of accepted claims, we can identify which of the initial responsibility claims are accepted. Based on these accepted claims, responsibility attribution under CRE is determined. This attribution may involve one or multiple parties responsible for the AI-induced incidents. It is important to note that CRE provides responsibility attribution in the sense of identifying who should be responsible, but it does not quantify the extent of their responsibility.

It is important to recognize that the CRE is a dynamic process requiring ongoing reflection and revision. This need arises for various reasons. Firstly, public preferences towards certain ethical claims may vary based on their perception of the claim under specific circumstances, leading to potential changes in the initial activation levels. Additionally, as AI technology advances, new claims or estimation methods may emerge, allowing for more informed computation based on additional evidence. Moreover, regulations or supporting principles may evolve over time. All these considerations emphasize that the CRE obtained is coherent and optimal only for the current moment. Therefore, it is essential to periodically revisit the initial stage and repeat the process whenever the current responsibility attribution appears inappropriate or requires adjustment.

\section{Properties}
In this section, we present several crucial properties of the CRE framework for responsibility attribution in AI-induced incidents.

\subsection{Traceability}

In this work, traceability in CRE is defined as the ability to trace responsibility attribution back to claims that lead to the decision. This creates an explicit audit trail that links outputs to inputs, preventing the results of CRE from becoming a pure black-box process.
Responsibility attribution derived from CRE is therefore traceable because it can be traced back to the claims and the relations embedded within the acceptance set. Initial claims capture intuitions about potential responsibility, while supportive claims reinforce them through principles, evidence, analogies, and ethical considerations. Together, these elements clarify why particular decisions are made and why opposing views are rejected, thereby enhancing transparency and promoting public trust in AI systems.
We acknowledge, however, that the traceability of CRE still has limitations. In large and complex networks, propagation paths may be difficult to present in a human-readable form, and concise explanations may require additional analysis tools. Moreover, the credibility of such paths depends on the transparency and rigor of the underlying network construction, including the identification of agents, claims, and support/conflict relations. When these inputs are well justified and documented, CRE can then provide a clear account of how and why responsibility is attributed.

\subsection{Coherence}

CRE is an equilibrium seeking to satisfy the maximum number of constraints while minimizing conflicts within the final decisions, aiming to establish a coherent system of beliefs that people would find acceptable. It is important to note that the result is not necessarily truth-conducive; instead, it tries to acquire justified beliefs under ethical constraints. This coherence goes beyond scientific proof or evidence, as it also takes into consideration people's preferences and acceptances towards a certain claim. Ultimately, the final judgment is made by humans, and accountability serves as a means to ensure satisfaction with the results, providing a balanced and coherent framework that resonates with the human perspective.

\subsection{Adaptivity}

The adaptivity of CRE is evident in its continuous revision and reflection process. When conditions change, the responsibility attribution can be promptly adjusted to align with the current situation. CRE is not a one-time iteration but an evolving procedure that encourages adaptation to challenges and diverse perspectives. While objections against RE in philosophy highlight its potential arbitrariness based on initial beliefs or intuitions, this characteristic proves advantageous in the context of responsibility attribution. Given the dynamic nature of situations, the ability to adapt equilibrium along the way is crucial. A fixed accountability framework could be exploited by malicious actors and would ultimately lack resonance with people's evolving understanding of the subject.

\begin{table*}[ht]
\centering
\fontsize{7}{9}\selectfont
\caption{Baseline Parameter Settings. The claim abbreviations correspond to those in Fig.~\ref{fig:network}.}
\label{tab:parameters}
\begin{tabular}{|l|*{16}{c|}}
\hline
 \rowcolor{Gray}
\textbf{Claim} & AGS & AGNS & AIDR & AIDNR & DR & DNR & AIR &  AINR & NR & DE & OM & DJW & SLOW & UT & UNFAIR\\
\hline
$a^0(u)$ & 0.2 & 0.1 & 0.01 & -0.01 & 0.01 & -0.01 & -0.2 & 0.1 & -0.2 & 0.1 & 0.01 & 0.2 & 0.01 & 0.01 & 0.01 \\
\hline
\rowcolor{Gray}
\textbf{Claim} &  INFO & BETTER & OIC & PRO & NON & OWN & RIGHT  & AIM&  ATT & AINM & LACK & SET & FIND & UBER & PRAC\\
\hline
$a^0(u)$  &  0.3 & 0.3 & 0.01 & 0.5 & 0.7 & 0.01 & 0.3 &  -0.3&  0 & 0.3 & 0.5 & -0.2 & 0 & 0.3 & 0.6\\
\hline
\end{tabular}
\end{table*}

\section{Case Study}

In this section, we use the medical decision-support system illustrated in Fig.~\ref{fig:network} as a case study to discuss responsibility attribution under different situations. We evaluate the performance of our CRE model using numerical experiments, implemented in a self-built Python simulator. The weight of each edge is set to $w(u,v)=1, \forall (u,v)\in E$, and the decay parameter $\gamma$ is set to be $0.05$. For the base case value, the initial activation levels are listed in Table~\ref{tab:parameters}. The claim abbreviations correspond to those in Fig.~\ref{fig:network}.

\subsection{Case $1$: Algorithm Design Error}

Consider a scenario where quantitative investigation reveals a significant design error in the algorithm. The claim authenticity for the algorithm design error (DE) is assessed to be $P_A(u)=0.9$. In the constraint network, we initialize the activation level of this claim as $a^0(\text{DE})=2P_A(u)-1=0.8$. Additionally, we assume that the society does not strongly agree with the idea that AI could be a moral agent, setting $a^0(\text{AIM})=-0.3$ and $a^0(\text{AINM})=0.3$ 

\begin{figure*}[!ht]
  \centering
   \subfloat[\small\textit{Initial}]{\includegraphics[width=0.495\linewidth]{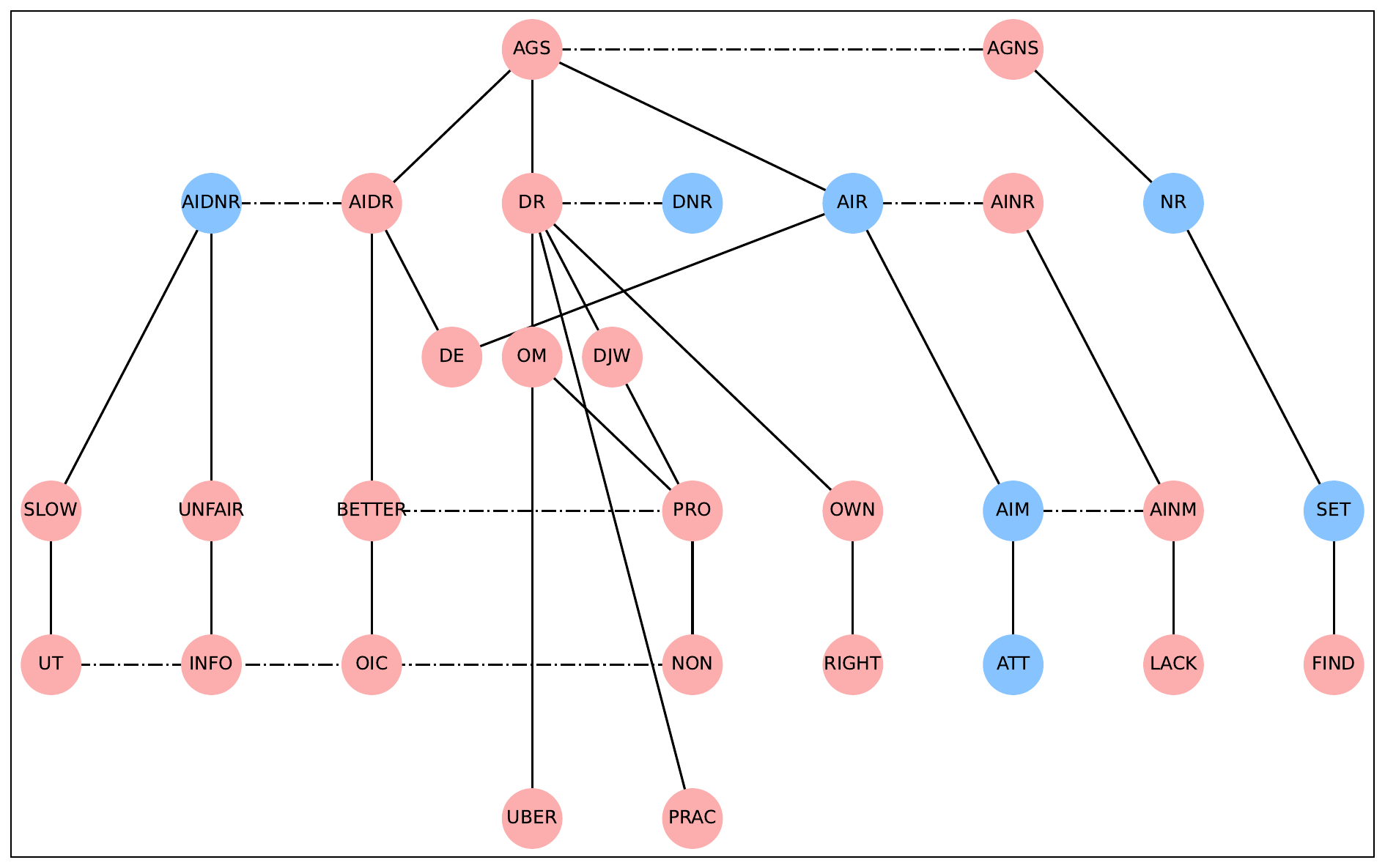}\label{fig:base}}
   \hfill
   \subfloat[\small\textit{Equilibrium}]{\includegraphics[width=0.495\linewidth]{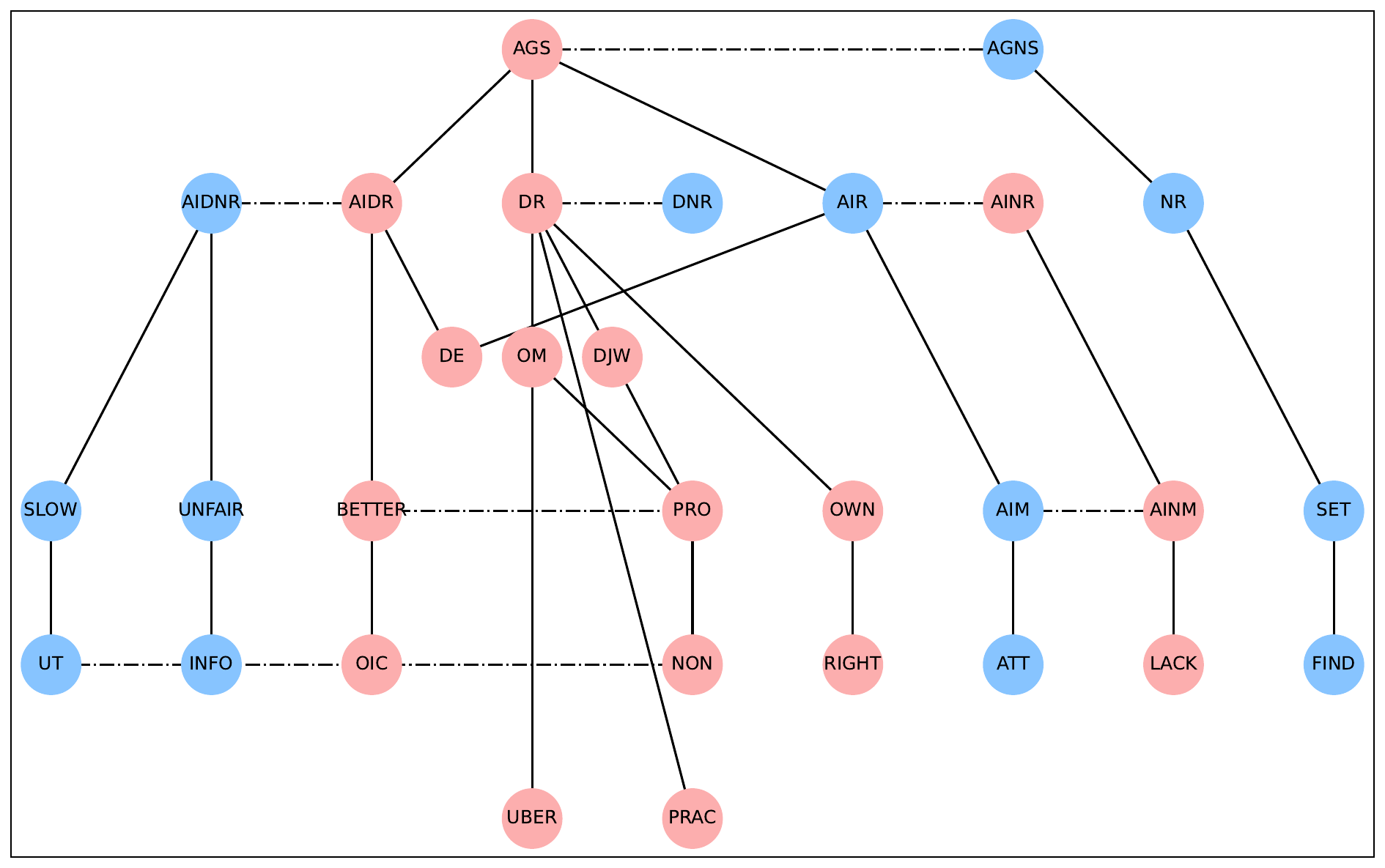}\label{fig:design}}
   \caption{Constraint network illustrating algorithm design errors. The node positions correspond to those in Fig.~\ref{fig:network}. Nodes colored in red represent accepted claims; any initial responsibility claim falling within this set is accepted for responsibility attribution. Nodes in blue represent rejected claims.}\label{fig:designcompare}
  \end{figure*}

Fig.~\ref{fig:designcompare} illustrate the initial claim status (Fig.~\ref{fig:base})and the claim status at the equilibrium (Fig.~\ref{fig:design}).
The evidence suggesting algorithm design error enhances the probability that the algorithm developer should be responsible for the incident, resulting in the rejection of the opposing claim that the algorithm developer should not be responsible. 
Furthermore, the analogy cases and ethical concerns indicate that the doctor should also share responsibility for the patient, even with the algorithm design error, as the doctor holds professional responsibility and should exercise careful judgment if the AI is wrong. 
Since we initialize the idea with the belief that AI should not be a moral agent, AI itself is not attributed responsibility in this case. Finally, the equilibrium suggests that people believe accountability should be assigned to individuals involved in the incident, rather than expecting society to bear the loss caused by such an AI system.

\subsection{Case $2$: Doctor Malpractice}

Next, we consider the scenario where the doctor has proved to have operational malpractice through investigation. In this case, we keep the base initial activation value in Table~\ref{tab:parameters} but set the initial activation level of the claim operational malpractice to $a^0(\text{OM})=0.6$ and the claim doctor's judgment is wrong to $a^0(\text{DJW})=0.2$.

\begin{figure}[!t]
    \centering
  \includegraphics[width=\linewidth]{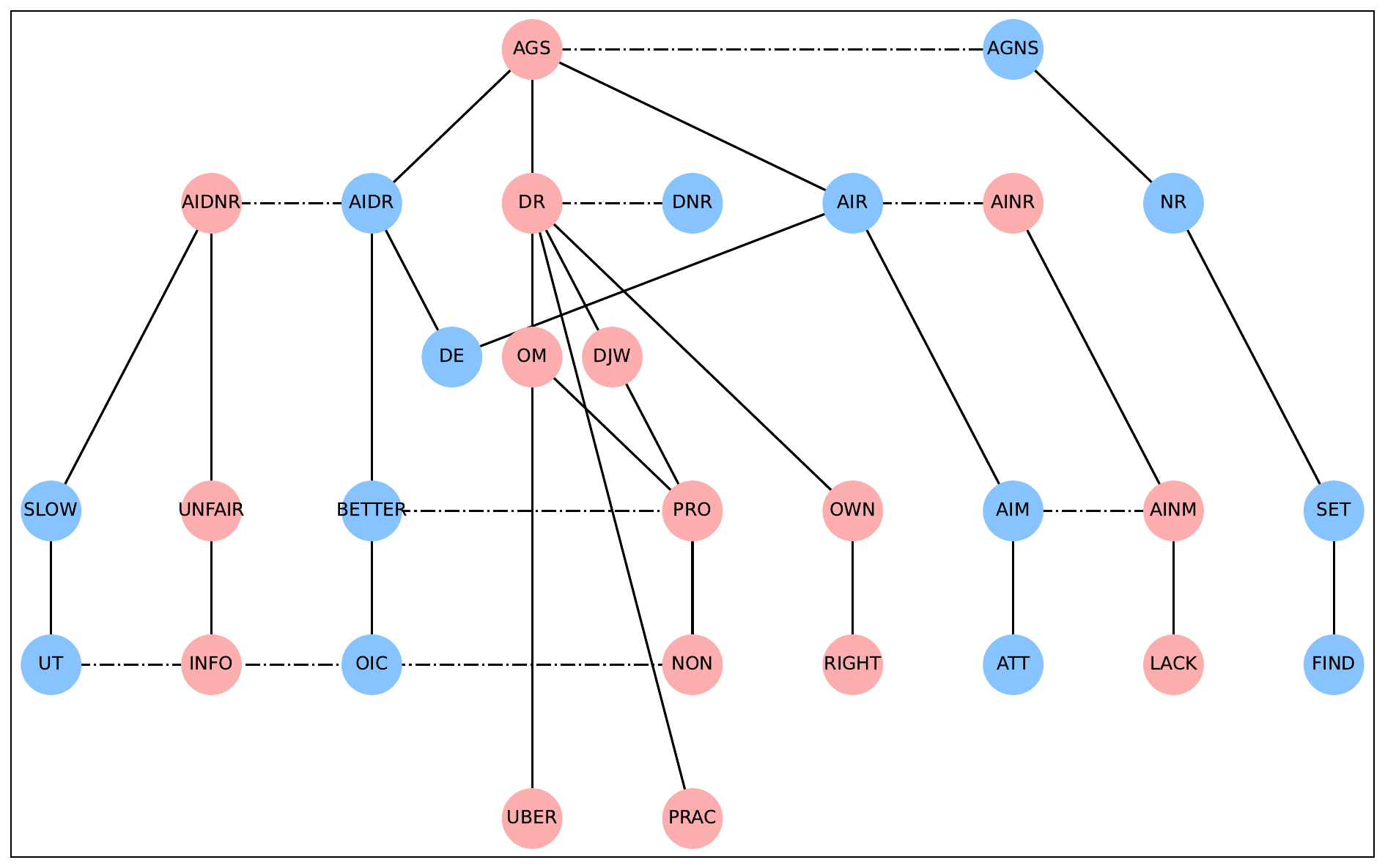}
  \caption{Constraint network with doctor operational malpractice at equilibrium.}
  \label{fig:operate}
\end{figure}

The division of the equilibrium claim set is illustrated in Fig.~\ref{fig:operate}. Unlike the previous case, only the doctor is determined to be responsible for the incident in this scenario. This is attributed to the high initial value of the operational malpractice claim. Additionally, the positive activation level of the analogy case provides support for the claim. In traditional medical operations, the doctor is responsible for any malpractice. Besides, in the Uber autonomous vehicle case, the responsibility is determined to be attributed to the test operator, instead of the algorithm developer. All these analogy cases provide legal support to the new area of AI-induced incidents, which is lacking regulation or laws for now.

\subsection{Case $3$: Social Belief in Collective Responsibility}

In the final scenario, we explore the societal belief that responsibility for AI-induced incidents should be collectively shared, diverging from the instinct to pinpoint an individual culprit. Despite challenging traditional intuitions, we examine the consequences of this collective responsibility attribution, envisioning a situation where such a belief is enforced by the government or other factors. In this setting, we set $a^0(\text{SET})$ and $a^0(\text{FIND})$ to be $0.8$.

As depicted in Fig.~\ref{fig:dis}, at the equilibrium, the claims accepted indicate that no individual should be held responsible for the incident. By embracing this perspective, society can collectively address the responsibility gap, proposing a shared responsibility model where funds are allocated for compensation. This approach provides a practical solution to dissolve the problem of individual responsibility while ensuring compensation in AI-induced incidents.

\begin{figure}[!t]
    \centering
  \includegraphics[width=\linewidth]{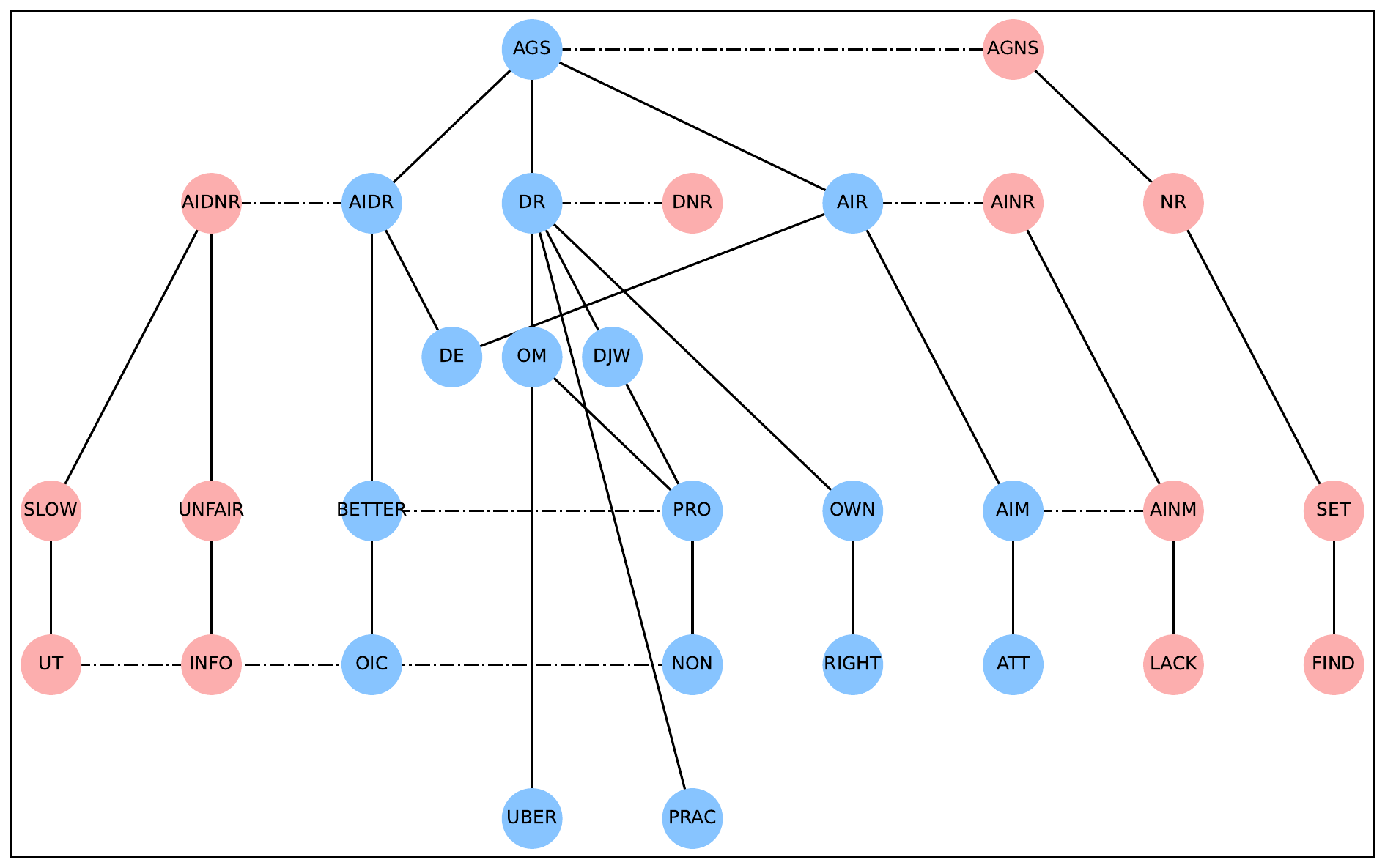}
  \caption{Constraint network with social belief in collective responsibility at equilibrium. }
  \label{fig:dis}
\end{figure}

\section{Conclusion}

Responsibility attribution and accountability in the realm of AI pose intricate challenges in today's society. In this study, we introduce a computational reflective equilibrium (CRE) approach to responsibility attribution in AI-induced incidents, with the goal of achieving a coherent and ethically acceptable equilibrium for all stakeholders. We describe the workflow of the CRE computation process and discuss the importance of the initial activation level in equilibrium computation. Using an AI-assisted medical decision-support system as an illustrative example, we demonstrate how different initialization leads to distinct responsibility attributions. Our framework offers traceability in responsibility attribution reasoning, coherence by ensuring mutual support among claims, and adaptability to varying ethical perspectives. It is essential to note that the framework is not necessarily truth-conducive but serves as a tool to provide ethically reasoned justifications that are coherently acceptable to all stakeholders. To enhance precision in responsibility attribution, continuous monitoring, revision, and reflection of the claims and activation levels are imperative for adjustments, contributing to the development of a more sustainable and enduring framework.

\bibliographystyle{IEEEtran}
\bibliography{ref}

\section*{Biography Section}

\begin{IEEEbiography}[{\includegraphics[width=1in,height=1.25in,clip,keepaspectratio]{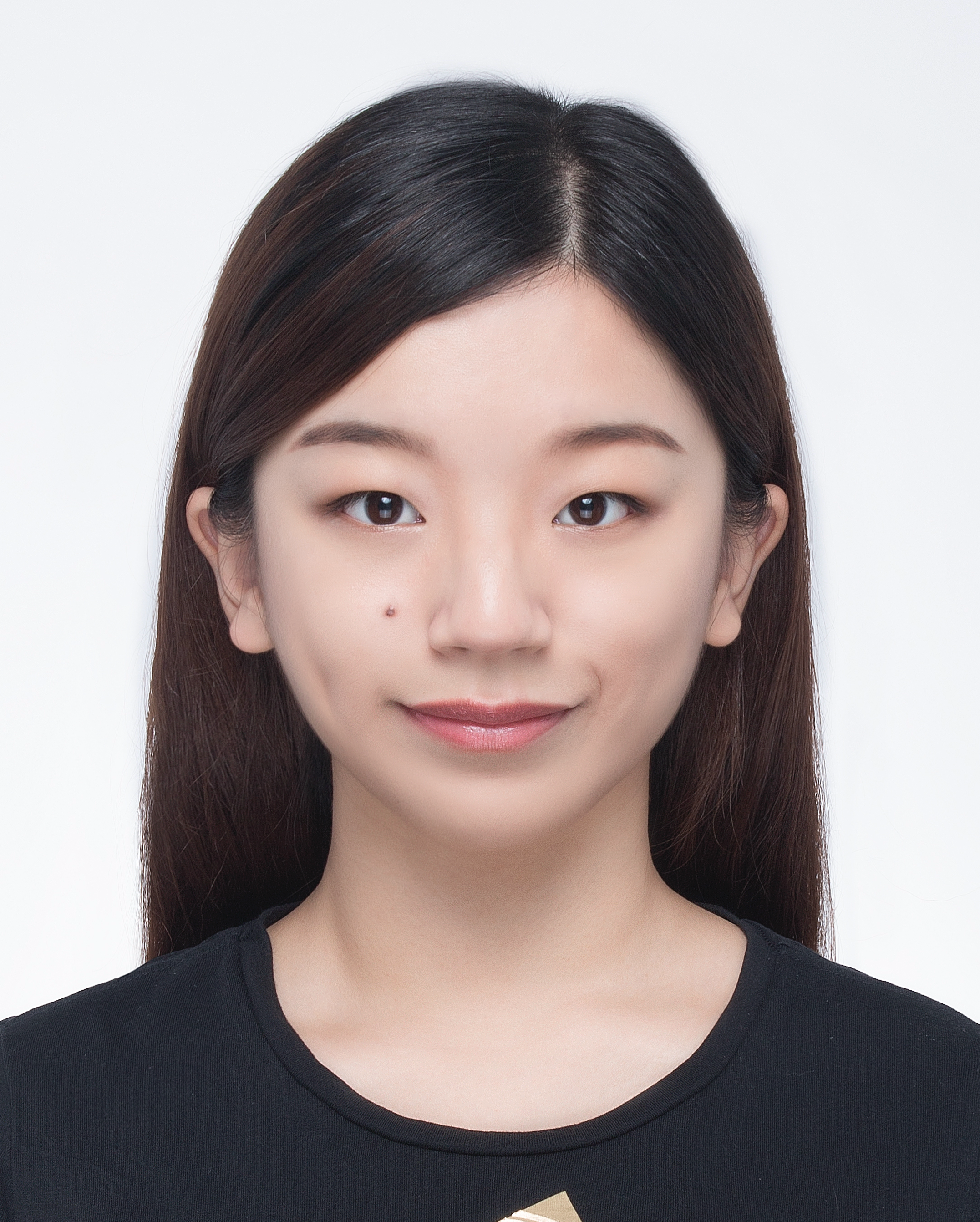}}]{Yunfei Ge} received her B.S. degree with Honors in Optoelectronics from Sun Yat-Sen University (SYSU), Guangzhou, China, in 2017, and her M.S. degree in Electrical Engineering from Columbia University, New York, NY, USA, in 2018. She entered the doctoral program in Electrical Engineering at New York University (NYU) Tandon School of Engineering in 2019. She was a recipient of the Honors Thesis Award for her undergraduate studies at SYSU.
During her time at NYU, she received the David C. and Cecilia M. Chang Education Award in 2021 for graduate teaching excellence in Electrical and Computer Engineering. Additionally, she was awarded the Li Publication Award in 2024 for her research published in the top venue, IEEE Transactions on Information Forensics and Security (TIFS).
\end{IEEEbiography}

\begin{IEEEbiography}[{\includegraphics[width=1in,height=1.25in,clip,keepaspectratio]{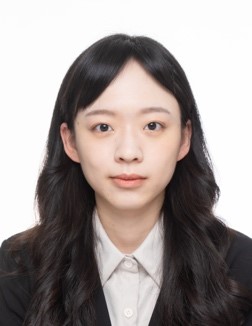}}]{Ya-Ting Yang} received her B.S. degree in electrical engineering from National Tsing Hua University, Hsinchu, Taiwan, in 2020, and the M.S. degree in communication engineering from National Taiwan University, Taipei, Taiwan, in 2022.  She is currently pursuing the Ph.D. degree in electrical and computer engineering with the Tandon School of Engineering, New York University (NYU), New York, NY, USA. Her current research interests include game theory and the impact of human and societal factors, such as cognitive biases, misinformation, and accountability issues, in cyber-physical systems.
\end{IEEEbiography}

\begin{IEEEbiography}[{\includegraphics[width=1in,height=1.25in,clip,keepaspectratio]{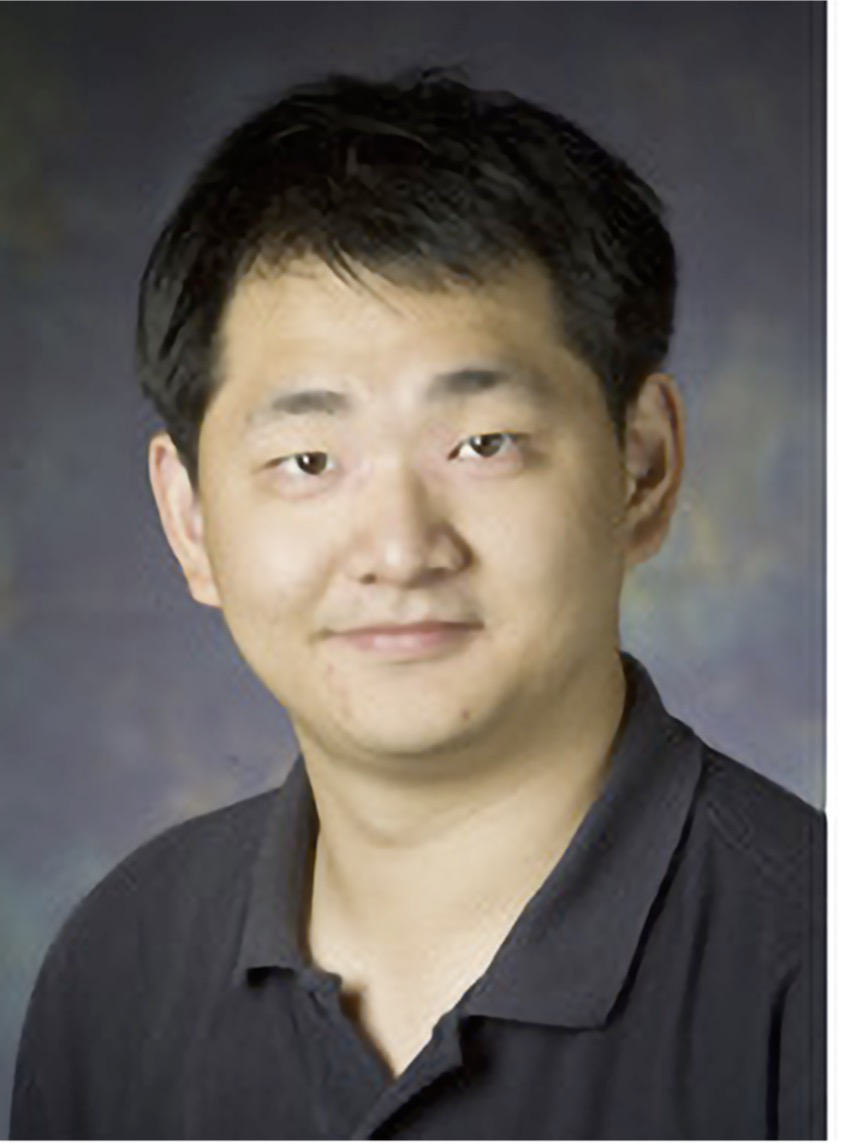}}]{Quanyan Zhu} Quanyan Zhu received B.Eng. in Honors Electrical Engineering from McGill University in 2006, M.A.
Sc. from the University of Toronto in 2008, and Ph.D. from the University of Illinois at Urbana-
Champaign (UIUC) in 2013. After stints at Princeton University, he is currently an associate professor at
the Department of Electrical and Computer Engineering, New York University (NYU). He is an affiliated
faculty member of the Center for Urban Science and Progress (CUSP) at NYU. He is a recipient of many
awards, including NSF CAREER Award, NYU Goddard Junior Faculty Fellowship, NSERC Postdoctoral
Fellowship (PDF), NSERC Canada Graduate Scholarship (CGS), and Mavis Future Faculty Fellowships.
He spearheaded and chaired INFOCOM Workshop on Communications and Control on Smart Energy
Systems (CCSES), Midwest Workshop on Control and Game Theory (WCGT), and ICRA workshop on
Security and Privacy of Robotics. His current research interests include game theory, machine learning,
cyber deception, network optimization and control, smart cities, Internet of Things, and cyber-physical
systems. He has served as the general chair or the TPC chair of the 7th and the 11th Conference on
Decision and Game Theory for Security (GameSec) in 2016 and 2020, the 9th International Conference
on NETwork Games, COntrol and OPtimisation (NETGCOOP) in 2018, the 5th International Conference
on Artificial Intelligence and Security (ICAIS 2019) in 2019, and 2020 IEEE Workshop on Information
Forensics and Security (WIFS). He has also spearheaded in 2020 the IEEE Control System Society (CSS)
Technical Committee on Security, Privacy, and Resilience. He is a co-author of two recent books
published by Springer: \textit{Cyber-Security in Critical Infrastructures: A Game-Theoretic Approach}
(with S. Rass, S. Schauer, and S. König) and \textit{A Game- and Decision-Theoretic Approach to
Resilient Interdependent Network Analysis and Design} (with J. Chen).
\end{IEEEbiography}

\vfill

\end{document}